\documentclass{article}
\usepackage[preprint]{neurips_2026}

\usepackage[utf8]{inputenc} % allow utf-8 input
\usepackage[T1]{fontenc}    % use 8-bit T1 fonts
\usepackage{hyperref}       % hyperlinks
\usepackage{url}            % simple URL typesetting
\usepackage{booktabs}       % professional-quality tables
\usepackage{amsfonts}       % blackboard math symbols
\usepackage{nicefrac}       % compact symbols for 1/2, etc.
\usepackage{xcolor}         % colors
\usepackage{amsmath}
\usepackage{amsthm}
\usepackage{graphicx}
\usepackage{amssymb}
\usepackage{algorithm}
\usepackage{algpseudocode}
\usepackage{float}
\usepackage{enumitem}
\usepackage{multirow}
\usepackage{wrapfig}
\usepackage{capt-of}

\DeclareFontShape{T1}{ptm}{m}{scit}{<->ssub*ptm/m/sc}{}{}

\newtheorem{definition}{Definition}
\newtheorem{proposition}{Proposition}

\theoremstyle{remark}

\title{State-Centric Decision Process}

\author{
    \begin{tabular}{ccc}
        \normalfont\mdseries SUNGHEON JEONG$^1$ &
        \normalfont\mdseries RYOZO MASUKAWA$^{1}$ &
        \normalfont\mdseries SANGGEON YUN$^1$
    \end{tabular}
    \\[0.7em]
    \begin{tabular}{cc}
        \normalfont\mdseries MAHDI IMANI$^2$ &
        \normalfont\mdseries MOHSEN IMANI$^1$
    \end{tabular}
    \\[1.0em]
    \begin{tabular}{c}
        $^{1}$University of California, Irvine, $^{2}$Northeastern University \\
        {\tt\small sungheoj@uci.edu}
    \end{tabular}
}

\begin{document}
    \maketitle

    \begin{abstract}
    Language environments such as web browsers, code terminals, and interactive simulations emit raw text rather than states, and provide none of the runtime structure that MDP analysis requires. No explicit state space, no observation-to-state mapping, no certified transitions, and no termination criterion. We introduce the State-Centric Decision Process (SDP), a runtime framework that constructs these missing inputs by having the agent build them, predicate by predicate, as it acts. At each step the agent commits to a natural-language predicate describing how the world should look, takes an action to make it true, and checks the observation against it. Predicates that pass become certified states, and the resulting trajectory carries the four objects language environments do not provide, namely a task-induced state space, an observation-to-state mapping, certified transitions, and a termination criterion. We evaluate SDP on five benchmarks spanning planning, scientific exploration, web reasoning, and multi-hop question answering. SDP achieves the best training-free results on all five, with the advantage widening as the horizon grows. The certified trajectories additionally support analyses unavailable to reactive agents, including per-predicate credit assignment, failure localization, partial-progress measurement, and modular operator replacement.
\end{abstract}
    \section{Introduction}
    \label{sec:intro}

    Language agents operate in environments that were never designed for autonomous decision-making, from web browsers and code terminals to interactive simulations and multi-step tool-use pipelines~\citep{yoran2024assistantbench, zhou2023webarena, jimenez2023swe, wang2022scienceworld, xie2024travelplanner}. Large language models have made this possible, absorbing enough world knowledge from vast training corpora to select reasonable actions from raw observations without environment-specific engineering~\citep{brown2020language, wei2022chain, achiam2023gpt, huang2022language, xi2025rise}. Yet success at action selection does not give the trajectory any formal structure: no explicit states, no verified transitions, nothing on which downstream methods can define transitions or assign credit. Language environments cannot supply this structure even in principle. The same situation admits unboundedly many valid natural-language descriptions, and which counts as a useful state is determined by the goal, a choice the environment has no access to. Without a state space, sequential decision-making has no surface to operate on.

    Existing language agents address parts of this gap but none closes it. Reactive agents~\citep{yao2022react, schick2023toolformer} interleave reasoning with action selection yet operate directly on raw observations without constructing an explicit state. Reflective agents go further, accumulating verbal lessons or causal memories across episodes~\citep{shinn2023reflexion, majumder2023clin, zhao2024expel}, but these summaries are open-ended text rather than states linked by certified transitions. Action planners~\citep{wang2023plan, yao2023tree, zhou2023language} deliberate over candidate action sequences before or during execution, gaining the benefit of lookahead, but the plan entries are things to do rather than conditions to verify, so progress cannot be checked against the environment. World-model approaches~\citep{wang2023describe, hao2023reasoning, liu2023reason, sun2023adaplanner} construct internal descriptions of the environment, moving closest to explicit state, but the descriptions are consumed by the same module that selects actions, with no per-step certification that they actually hold. In every case the trajectory remains a sequence of raw observations and actions, never a sequence of verified states over which transitions and credit can be defined. On short tasks this is tolerable; on long-horizon problems, where errors compound and intermediate progress must be tracked, agents have no formal signal that progress is being made. 

    We propose to close this gap by having the agent construct its own Markov Decision Process (MDP) at runtime. We call the resulting framework the \emph{State-Centric Decision Process} (SDP). Instead of choosing actions and treating state as a byproduct, the agent first commits to a natural-language predicate describing how the world should look after the next action, a checkable condition on the resulting observation. It then takes an action intended to make that predicate true and checks the observation against it. Predicates that pass this check become certified states. By committing to a predicate before acting, the agent forces its intent into a form the environment can falsify, and the resulting trajectory carries precisely the four objects that MDP analysis requires but language environments do not provide: states, transitions, actions, and credit. SDP is therefore not a competitor to MDP-based formulations of agency, but the interface layer they presuppose.

        \textbf{Contributions.}
        \begin{enumerate}[leftmargin=1.8em, topsep=2pt, parsep=0pt]
            
            \item \textbf{Identifying the missing inputs.} We formalize a specification problem unstated in the literature: MDP-based analysis requires four objects no language environment supplies, and the gap is one of specification rather than sample complexity.

            \item \textbf{The SDP framework.} We introduce the State-Centric Decision Process, a runtime framework producing Markov trajectories by decomposing agency into operators \textsc{Propose}, \textsc{Realize}, \textsc{Validate}, and \textsc{Replan} over natural-language predicates.

            \item \textbf{Empirical evaluation.} SDP achieves the best training-free results on all five benchmarks, with the gap widening as task horizon grows, and ablations isolating each operator's contribution.
            
            \item \textbf{Trajectory as diagnostic artifact.} Certified trajectories support analyses unavailable to prior agents: failure localization, partial-progress tracking, cascade recording, and validator auditing.    	
        \end{enumerate}
    \section{Preliminaries: the MDP gap in language environments}
	\label{sec:missing-inputs}
	
    MDP analysis requires a state space, an observation-to-state mapping, a transition kernel, and a termination criterion~\citep{puterman2014markov, sutton1998reinforcement}. Language environments provide none of these. They emit unstructured text such as web pages, terminal outputs, and API responses rather than states, and their interface does not commit to what the agent should treat as state~\cite{xi2025rise, wang2024survey, sumers2023cognitive}. Without a fixed state space the remaining MDP constructs cannot even be stated, let alone estimated or optimized. This section identifies the gap and the four inputs any framework must supply before MDP analysis.
    
    \textbf{Useful state abstractions are goal-dependent.}
		Let $\mathcal{O}$ denote the space of raw environment outputs and $\mathcal{H} = \bigcup_{t \ge 0}(\mathcal{O} \times A)^t \times \mathcal{O}$ the space of interaction histories. An MDP state is the image of a history under an abstraction $\phi : \mathcal{H} \to S$ that preserves the Markov property~\citep{li2006towards, givan2003equivalence}. For a single $\phi$ to suffice, the distinctions it draws must be adequate for every goal the agent might pursue. In language environments this fails, since the distinctions that matter depend on the goal. Two histories that a summarization policy can safely identify must be separated by a checkout policy, and vice versa. Any $\phi$ fine enough to be Markov across all goals collapses toward the identity on $\mathcal{H}$, while any coarser $\phi$ is goal-specific. The useful abstraction is therefore not a single $\phi$ but a family $\{\phi_g\}_{g \in \mathcal{G}}$ indexed by goals~\citep{abel2018state, andrychowicz2017hindsight, schaul2015universal}, and the MDP formalism provides no mechanism to select one at runtime.

    \textbf{Once the state space goes, the rest follows.}
        A state space that varies with the goal makes the transition kernel $T : S \times A \to \Delta(S)$ undefined as a mathematical object, because its domain and codomain are not fixed sets. This is not a problem more data could resolve; function approximation requires a target object to approximate, and there is none. Every construct built on $T$, including value functions, Bellman backups, and policy gradients~\citep{sutton1998reinforcement, williams1992simple}, inherits the same gap. A termination criterion faces the same absence, since the goal is supplied with the task rather than the environment, and nothing in the raw output stream signals when the task is done. Two responses might be expected, and neither works. Fixing some maximally fine $\phi$ and treating the resulting space as truth is what the POMDP relaxation effectively does~\citep{kaelbling1998planning, young2013pomdp, murphy2000survey}, but its filtering equations presuppose precisely the $S$ and $T$ whose existence is in question. Skipping $\phi$ entirely and letting a neural network operate on raw history works empirically and is the dominant approach in language agents~\citep{yao2022react, wang2023voyager}, but it recovers no analytic object on which transitions, value, or progress can be defined.

	\textbf{The four missing inputs.}
		The preceding argument exposes four inputs that MDP analysis requires but that no language environment provides:
		\begin{enumerate}[leftmargin=1.8em, itemsep=2pt, topsep=4pt]
			\item \textbf{State space.} A set $S$ over which policies, values, and transitions are defined. No fixed $S$ exists because the useful abstraction is goal-indexed.
			\item \textbf{Observation-to-state mapping.} A function that turns each new observation into a state update. No single $\phi$ works across goals.
			\item \textbf{Certified transitions.} Tuples $(s, a, s')$ whose validity has been checked, not merely assumed from temporal adjacency. Without a shared $S$ there is no space in which to express them.
            \item \textbf{Termination criterion.} A predicate over the state space that signals task completion. Language environments emit no such signal; the goal is supplied with the task, not by the environment.
		\end{enumerate}
    \section{Method: State-Centric Decision Process}
	\label{sec:method}
	
    We introduce the \emph{State-Centric Decision Process} (SDP), a runtime framework that supplies the four missing inputs by having the agent construct them, predicate by predicate, as it acts. The constructed states are not observations, not summaries of observations, and not latent vectors. They are natural-language predicates that the agent commits to before acting, each describing how the world should look at a future step. The trajectory that results carries precisely what language environments do not provide, namely a state space populated by the task, a mapping from observations to state updates, transitions whose validity has been checked, and a termination test grounded in the goal (Figure~\ref{fig:sdp-structure}).

	\subsection{The State-Centric Decision Process}
		\label{sec:sdp-def}

        The framework rests on four operators over a state space $\Sigma$, formalized below.
        
		\begin{definition}[State-Centric Decision Process]
		\label{def:sdp}
		An SDP is a tuple 
		\[
			(\Sigma, A, \mathcal{O}, \mathcal{T}, g, \textsc{Propose}, \textsc{Realize}, \textsc{Validate}, \textsc{Replan}), 
		\]
		where $\Sigma$ is a space of natural-language predicates over $\mathcal{O}$, $A$ is the action space, $\mathcal{O}$ is the space of raw environment outputs, $\mathcal{T}$ is the space of certified trajectories accumulated during execution, and $g \in \Sigma$ is a goal predicate given with the task. The four functions are:
		\begin{align}
		\textsc{Propose}  &: \Sigma \times \Sigma \to \Sigma, & (s_t, g) &\mapsto \hat{s}_{t+1}, \label{eq:propose} \\
		\textsc{Realize}  &: \Sigma \times \Sigma \to A, & (s_t, \hat{s}_{t+1}) &\mapsto a_t, \label{eq:realize} \\
		\textsc{Validate} &: \Sigma^* \times \mathcal{O} \to \mathbb{N}, & (\hat{s}_{t+1}, \ldots, \hat{s}_n;\, o) &\mapsto k, \label{eq:validate} \\
		\textsc{Replan}   &: \Sigma \times \Sigma \times \mathcal{T} \to \Sigma^*, & (s_t, g, \tau_t) &\mapsto (\hat{s}_{t+1}, \ldots, \hat{s}_n). \label{eq:replan}
		\end{align}
        \textsc{Propose} (Eq.~\ref{eq:propose}) sets the next target from the current state and the goal. \textsc{Realize} (Eq.~\ref{eq:realize}) selects an action to make the target hold. \textsc{Validate} (Eq.~\ref{eq:validate}) takes the resulting observation and returns the number $k \geq 0$ of consecutive targets from $\hat{s}_{t+1}$ that the observation satisfies, with $k = 0$ meaning the immediate target is unmet. \textsc{Replan} (Eq.~\ref{eq:replan}) produces a new target sequence from the current state to $g$ when the plan is unrecoverable. Here $n$ is the current plan length.
		\end{definition}
        By construction, only \textsc{Validate} consumes raw observations; \textsc{Propose} and \textsc{Realize} operate entirely on $\Sigma$, so targets reflect the agent's intent rather than reactions to whatever the environment presents. The three operators in the normal loop each read only local inputs, never the preceding history. \textsc{Replan} is the sole exception, receiving $\tau_t$ as the recovery mechanism invoked when the normal loop cannot advance. This locality yields the Markov property (Proposition~\ref{prop:markov}).
                        
        \textbf{Optimizing over states, not actions.}
			Definition~\ref{def:sdp} reorganizes the decision problem. A reactive agent solves $a_t^* = \arg\max_{a \in A} P(\text{success} \mid h_t, a)$ at each step, with $A$ as its decision space. SDP takes the \emph{state plan} as decision variable, decomposing the problem into two coupled stages:
			\begin{align}
				(\hat{s}_1, \ldots, \hat{s}_n)^* &= \arg\max_{(\hat{s}_1, \ldots, \hat{s}_n) \in \Sigma^n} \; P\bigl(\text{plan reaches } g \,\big|\, s_0\bigr), \label{eq:outer} \\
				a_t^* &= \arg\max_{a \in A} \; P\bigl(\textsc{Validate}(\hat{s}_{t+1}, \textsc{Env}(a)) \geq 1 \,\big|\, s_t, \hat{s}_{t+1}\bigr). \label{eq:inner}
			\end{align}
			The outer objective (Eq.~\ref{eq:outer}) searches over $\Sigma^n$ for a predicate chain from $s_0$ to $g$. The inner objective (Eq.~\ref{eq:inner}) searches over $A$ for an action whose environment response $o_{t+1}$ satisfies the next predicate. \textsc{Propose} commits to the outer choice before execution begins, \textsc{Realize} solves the inner problem at each step, and \textsc{Validate} supplies the per-step signal that closes the loop. In this paper, both stages are realized through prompted LLM calls rather than explicit optimization. \textsc{Replan} exists to repair cases where the resulting plan turns out to be unreachable. Making the optimization explicit is the subject of the research program SDP is intended to enable, discussed in Section~\ref{sec:discussion}.

        \begin{figure*}[t]
    		\centering
    		\includegraphics[width=0.9\textwidth]{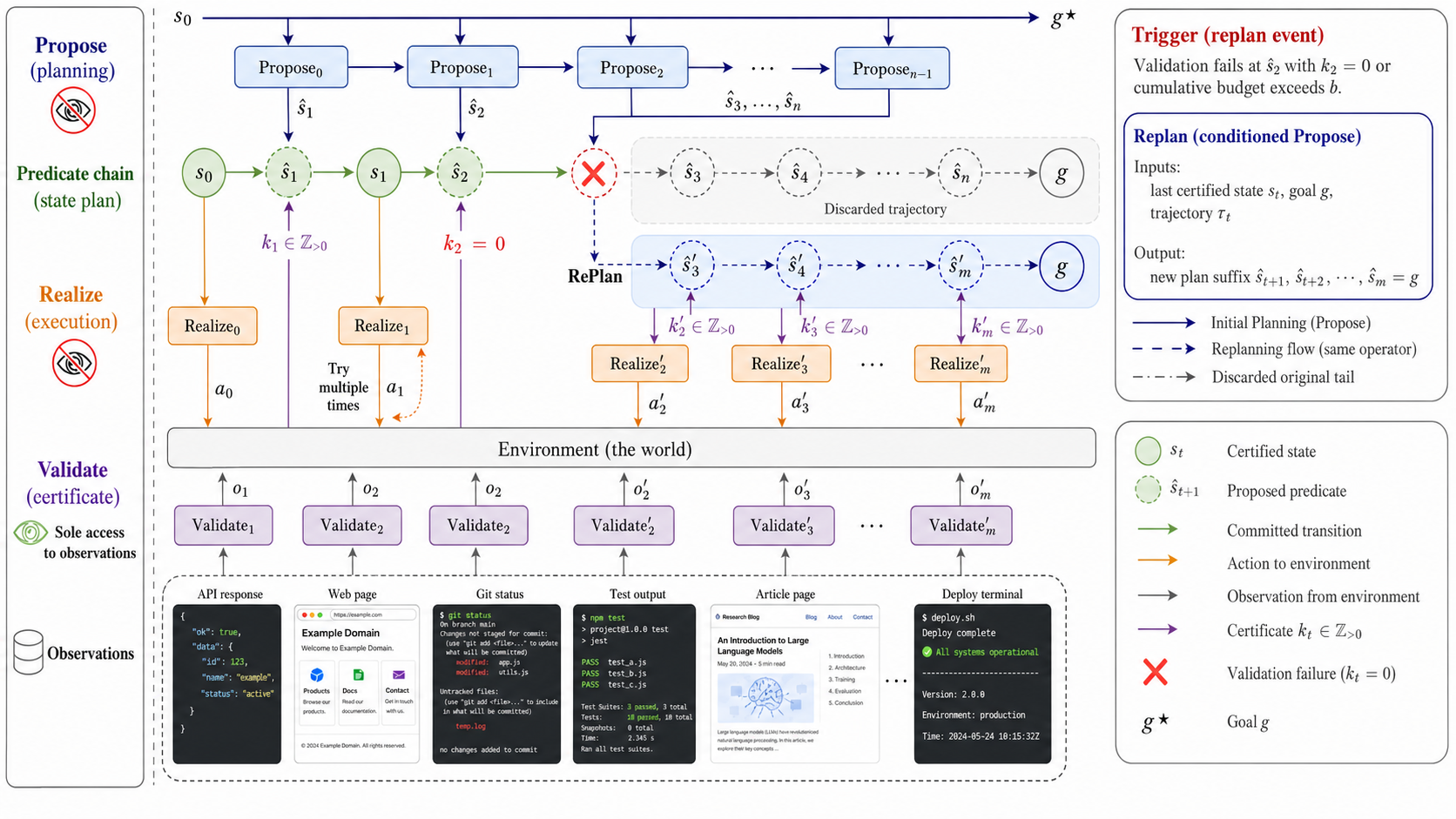}
            \vspace{-1.2em}
            \caption{\textbf{The SDP execution structure.} \textsc{Propose} builds the predicate chain from $s_0$ to $g$. \textsc{Realize} acts toward the next predicate. \textsc{Validate} checks the observation and is the sole interface for $\mathcal{O}$. When failures at a target exceed $b$, \textsc{Replan} discards and builds a new state.}
            \label{fig:sdp-structure}
    	\end{figure*}

    \subsection{The execution loop}
		\label{sec:execution}

        The SDP execution loop separates planning from stepping. The agent first applies \textsc{Propose} from $s_0$ and feeds each result back into \textsc{Propose} until it reaches $g$, producing a state plan $(\hat{s}_1, \ldots, \hat{s}_n = g)$. Algorithm~\ref{alg:sdp} gives the complete loop. At each step the agent selects an action via \textsc{Realize} and sends the environment's response to \textsc{Validate} (L.~4--6). If validation certifies one or more predicates, the cursor advances and the certified states enter the trajectory (L.~7--9). Otherwise the agent retries the same target, and after $b$ consecutive failures \textsc{Replan} replaces the remaining plan suffix (L.~10--14).

        \textbf{Cascade. One action can certify several targets.}
        	\textsc{Validate} returns an integer rather than a binary verdict because a single action sometimes satisfies multiple consecutive predicates at once. Such cascades arise from LLM knowledge spanning multiple sub-goals, predicate overlap, or overly fine-grained decomposition by \textsc{Propose}. For example, if the next two targets are ``the athlete's average pace is computed'' and ``the total distance is derived from that pace,'' a single Python snippet can satisfy both simultaneously. Treating the first target alone as satisfied would force a redundant second action whose outcome is already in hand. \textsc{Validate} therefore checks the observation against the head of the remaining plan, reports how many consecutive predicates are jointly satisfied, and advances the cursor by that many positions. The certified transitions take the form $(s_t, a_t, s_{t+k})$, generalizing single-step MDP transitions (Figure~\ref{fig:sdp-examples}).

\begin{figure*}[t]
\centering

\begin{minipage}{0.90\textwidth}
\centering

% ===================== Left: Algorithm =====================
\begin{minipage}[t]{0.48\linewidth}
\vspace{0pt}
\raggedright
\refstepcounter{algorithm}
\label{alg:sdp}

\hrule height 0.8pt
\vspace{2pt}
\noindent\textbf{Algorithm \thealgorithm} The SDP execution loop.
\vspace{2pt}
\hrule height 0.4pt
\vspace{3pt}

\scriptsize
\begin{algorithmic}[1]
\Require initial state $s_0$, goal $g$, attempt budget $b$
\State $(\hat{s}_1, \ldots, \hat{s}_n) \gets \textsc{BuildPlan}(s_0, g)$
\Statex \hfill $\triangleright$ Propose until $\hat{s}_n = g$
\State $\tau \gets (s_0)$, $\;t \gets 0$, $\;r \gets 0$
\While{$t < n$}
    \State $a_t \gets \textsc{Realize}(s_t, \hat{s}_{t+1})$
    \State $o_{t+1} \gets \textsc{Env}(a_t)$
    \State $k \gets \textsc{Validate}((\hat{s}_{t+1}, \ldots, \hat{s}_n), o_{t+1})$
    \If{$k \geq 1$}
        \State $s_{t+1:t+k} \gets \hat{s}_{t+1:t+k}$
        \Statex \hfill $\triangleright$ certify $k$ predicates
        \State $\tau \gets \tau \cup \{(a_t, s_{t+k})\}$
        \State $t \gets t + k$, $\;r \gets 0$
    \Else
        \State $r \gets r + 1$
        \If{$r > b$}
            \State $(\hat{s}_{t+1}, \ldots, \hat{s}_n) \gets \textsc{Replan}(s_t, g, \tau)$
            \State $r \gets 0$
        \EndIf
    \EndIf
\EndWhile
\State \Return $\tau$
\end{algorithmic}

\vspace{3pt}
\hrule height 0.8pt
\end{minipage}
\hfill
% ===================== Right: Figure =====================
\begin{minipage}[t]{0.48\linewidth}
\vspace{0pt}
\centering
\includegraphics[width=0.96\linewidth]{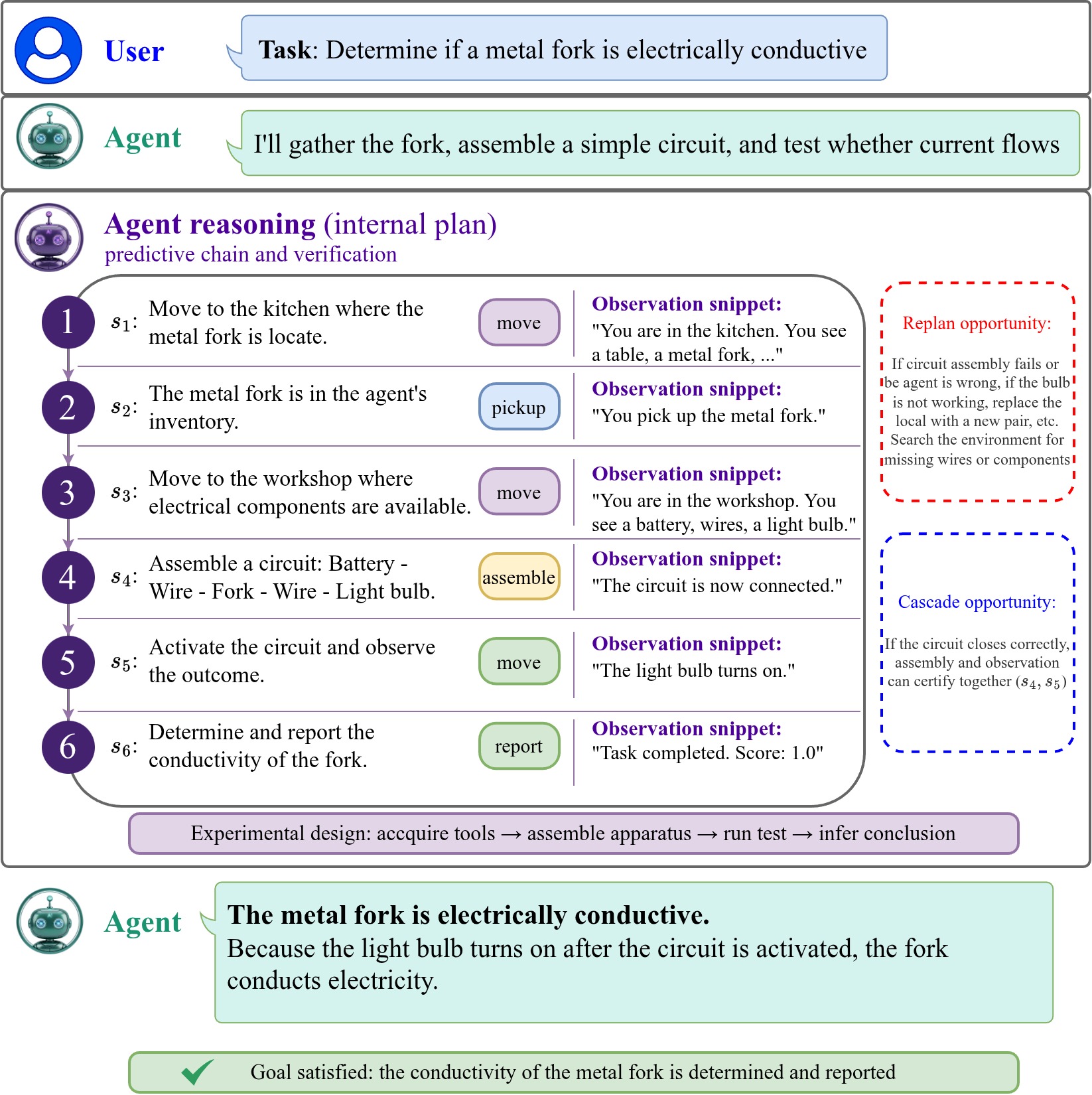}
\captionof{figure}{Example SDP run, with a cascade ($k \geq 2$) and a replan opportunity.}
\label{fig:sdp-examples}
\end{minipage}

\end{minipage}

\vspace{-0.5em}
\end{figure*}
 
		\paragraph{Failure in execution and failure in planning are corrected separately.}
            When an action fails to realize a target, \textsc{Realize} retries with a different action against the same predicate. The rest of the plan is untouched. The plan itself is revised only when repeated attempts at the same target exhaust the budget $b$. \textsc{Replan} then discards the plan tail and proposes a new sequence from $s_t$ to $g$ (Figure~\ref{fig:sdp-structure}). This separation matters most in long-horizon tasks. A broken action consumes one attempt and the agent retries, while a broken plan is locally repaired without restarting from $s_0$. Reactive agents must reconsider both whenever either fails. SDP corrects each on its own timescale.
                    
    \subsection{What an SDP run produces}
		\label{sec:sdp-analysis}

		An SDP run terminates with more than an answer to the task. It returns a structured artifact that records every certified state, the action that produced it, and the cascade depth recorded by each validation. The artifact has the form
		\[
		\bigl(s_0,\; a_0,\; s_{t_1},\; a_{t_1},\; \ldots,\; a_{T-1},\; s_T\bigr), \qquad s_T \models g,
		\]
		where the indices $t_1 < t_2 < \cdots < T$ mark the steps at which validation discharged at least one target. Each gap $t_{i+1} - t_i$ is the cascade depth $k$ at that step. This artifact closes the gap identified in Section~\ref{sec:missing-inputs} by supplying, item by item, the four objects that MDP analysis requires but that no language environment provides. Table~\ref{tab:supplies} states the correspondence explicitly.

\begin{wraptable}{r}{0.5\textwidth}
\vspace{-1.8em}
\centering
\caption{MDP inputs absent from language environments and the artifacts SDP produces. The last row adds a credit signal beyond the four.}
\label{tab:supplies}
\small
\begin{tabular}{@{}lp{0.32\textwidth}@{}}
\toprule
\textbf{MDP input} & \textbf{Supplied by SDP} \\
\midrule
$S$ & $\Sigma_g \subseteq \Sigma$, populated by the plan and cascades \\
\midrule
$\phi : \mathcal{O} \to S$ & \textsc{Validate}, mapping each observation to $k \in \mathbb{N}$ \\
\midrule
Transitions & Certified $(s_t, a_t, s_{t+k})$ from each successful validation \\
\midrule
Termination & The relation $s_T \models g$ \\
\midrule
Credit & The integer $k$ returned by \textsc{Validate} \\
\bottomrule
\end{tabular}
\end{wraptable}

        Table~\ref{tab:supplies} should not be read as a claim that we have applied any particular downstream method to the artifact. We have not. The claim is that each downstream method is now a \emph{well-posed problem} on this artifact, in a way it is not on raw histories of language environments. Making these problems well-posed is the contribution of this paper. Translating them into algorithms is the substance of the program described in Section~\ref{sec:discussion}. The artifact also carries a structural guarantee.
    
        \begin{proposition}[Markov property of SDP trajectories]
        \label{prop:markov}
        Let $\tau = (s_0, a_0, s_{t_1}, a_{t_1}, \ldots, s_{t_i})$ be a prefix of a trajectory generated by Algorithm~\ref{alg:sdp}, let $\text{prefix} = (s_0, a_0, \ldots, a_{t_i - 1})$ denote the history before $s_{t_i}$, and let $P_i = (\hat{s}_{t_i+1}, \ldots, \hat{s}_n)$ be the plan tail at step $t_i$. Assume the environment's response $P(o_{t_i+1} \mid a_{t_i}, \text{prefix})$ depends on the prefix only through $(s_{t_i}, P_i)$. Then the next certified state satisfies
        \begin{equation}
        \label{eq:markov}
        P\bigl(s_{t_{i+1}} \,\big|\, s_{t_i}, P_i, \text{prefix}\bigr) = P\bigl(s_{t_{i+1}} \,\big|\, s_{t_i}, P_i\bigr).
        \end{equation}
        \end{proposition}
        
        \begin{proof}[Proof sketch]
        The next certified state $s_{t_{i+1}}$ depends only on $a_{t_i}$ and the value $k = \textsc{Validate}(P_i, o_{t_i + 1})$. By Definition~\ref{def:sdp}, $\textsc{Realize}(s_{t_i}, \hat{s}_{t_i + 1})$ reads only the current state and the next target, and \textsc{Validate} reads only the plan tail $P_i$ and the new observation. Neither consumes the prefix, and $o_{t_i+1}$ depends on the prefix only through $(s_{t_i}, P_i)$ by assumption. \textsc{Replan}, when invoked, receives $\tau$ only to produce a new plan tail $P_i'$, which then plays the role of $P_i$ for subsequent steps, and the same independence holds with $P_i'$ in place of $P_i$. The prefix therefore enters $s_{t_{i+1}}$ only through $(s_{t_i}, P_i)$.
        \end{proof}

        The Markov property here follows from how state is defined, given a mild conditional independence assumption on the environment, rather than from an empirical claim about the environment's full dynamics. SDP states are predicates the agent has certified to hold, not environment configurations whose latent dynamics might leak across step boundaries. Trajectories with the same certified state and remaining plan are indistinguishable for predicting the next certified state, which depends only on what the agent intends to certify next and the action it tries. Proposition~\ref{prop:markov} establishes exactly the conditional independence of Eq.~\ref{eq:markov}, with the pair $(s_{t_i}, P_i)$ playing the role of the Markov state.

    \section{Experiments}
    \label{sec:experiments}
    
    We evaluate SDP along two axes. First, whether organizing an agent around a state plan improves task performance (Section~\ref{sec:main-results}). Second, whether the certified trajectories support analyses unavailable to reactive agents (Section~\ref{sec:anatomy}). Section~\ref{sec:ablation} isolates the contribution of each loop-level mechanism via ablation. The five benchmarks span distinct combinations of environment structure and goal type, including constraint satisfaction over structured option sets (TravelPlanner~\citep{xie2024travelplanner}), open-ended web reasoning with free-form retrieval (AssistantBench~\citep{yoran2024assistantbench}), interactive scientific exploration in a text-based simulator (ScienceWorld~\citep{wang2022scienceworld}), and multi-hop question answering at varying chain depths (HotpotQA~\citep{yang2018hotpotqa}, MuSiQue~\citep{trivedi2022musique}). For each benchmark we compare SDP against the published baselines evaluated on it, retaining each baseline with the LLM reported in its original paper. We report benchmark-selection rationale and baseline fairness criteria, including LLM backbone matching and sample size alignment, in Appendices~\ref{app:benchmarks} and~\ref{app:implementation_details}.
    
    \subsection{Task performance across language environments}
        \label{sec:main-results} 

        \textbf{TravelPlanner.}
            Table~\ref{tab:travelplanner} reports results on TravelPlanner~\citep{xie2024travelplanner}. SDP's largest gain is on hard constraints, achieving 97.4\% Micro and 93.8\% Macro and exceeding ATLAS~\citep{choi2025atlas} by 14.8 and 19.4 points despite ATLAS using a larger backbone (Gemini-2.5-Pro). This advantage traces to SDP's predicate structure: each constraint (budget, accommodation, transportation) becomes a separate predicate that must be certified before the plan advances, so violations are caught at the step where they arise rather than discovered after assembly. Two failure modes of plan-as-text baselines reinforce this gap. In our reproduction runs, 12--18\% of outputs from ReAct, Plan-and-Solve, and Reflexion were malformed, and 20--30\% of feasible tasks suffered budget overflow from single-pass generation. SDP eliminates both, since each \textsc{Realize} call receives only constraint-satisfying options and the LLM selects among guaranteed-feasible candidates rather than checking constraints itself.

\begin{table*}[h]
\centering
\caption{Results on TravelPlanner~\cite{xie2024travelplanner}. Metrics are Delivery Rate, Commonsense constraint satisfaction (Micro/Macro), Hard Constraint satisfaction (Micro/Macro), and Final Pass Rate (\%).}
\label{tab:travelplanner}
\small
\setlength{\tabcolsep}{4pt}
\resizebox{0.80\textwidth}{!}{%
\begin{tabular}{@{}llcccccc@{}}
\toprule
& & & \multicolumn{2}{c}{Commonsense} & \multicolumn{2}{c}{Hard Constraint} & \\
\cmidrule(lr){4-5} \cmidrule(lr){6-7}
Algorithm & LLM & Delivery Rate & Micro & Macro & Micro & Macro & Final Pass Rate \\
\midrule
ReAct~\cite{yao2022react}                 & Gemini-3.1-flash-lite  & 95.0  & 78.6 & 16.1 & 83.8 & 61.7 & \phantom{0}7.2 \\
Reflexion~\cite{shinn2023reflexion}       & Gemini-3.1-flash-lite  & 100.0 & 80.5 & 26.7 & 77.9 & 55.6 & 12.8 \\
Plan-and-Solve~\cite{wang2023plan}        & Gemini-3.1-flash-lite  & \phantom{0}99.4  & 84.7 & 26.1 & 78.9 & 56.7 & 13.3 \\
MIRROR~\cite{guo2025mirror}               & GPT-4o                 & 100.0 & 73.2 & 13.9 & 27.6 & 13.3 & \phantom{0}2.2 \\
GoalAct~\cite{chen2025enhancing}          & Gemini-3.1-flash-lite  & \phantom{0}99.4  & 83.1 & 25.0 & 78.7 & 56.1 & 15.0 \\
TDP~\cite{li2026beyond}                   & Gemini-3.1-flash-lite  & \phantom{0}99.4  & 81.6 & 27.2 & 79.2 & 61.1 & 13.3 \\
EvoAgent~\cite{yuan2025evoagent}          & Gemini-2.5-Pro         & 100.0 & 78.1 & 23.9 & 57.9 & 40.6 & 12.2 \\
PMC~\cite{zhang2025planning}              & Gemini-2.5-Pro         & 100.0 & 78.7 & 30.6 & 43.3 & 37.2 & 23.3 \\
ATLAS~\cite{choi2025atlas}                & Gemini-2.5-Pro         & 100.0 & 88.5 & 48.3 & 82.6 & 74.4 & 44.4 \\
\midrule
\textbf{SDP (Ours)} & Gemini-3.1-flash-lite  & \textbf{100.0} & \textbf{95.2} & 61.7 & \textbf{97.4} & 93.8 & 61.7 \\
\textbf{SDP (Ours)} & GPT-4o                 & \textbf{100.0} & \textbf{95.6} & \textbf{65.6} & \textbf{97.4} & \textbf{93.9} & \textbf{65.6} \\
\bottomrule
\end{tabular}%
}
\end{table*}

        \textbf{AssistantBench.}
            Table~\ref{tab:assistantbench} reports results on AssistantBench~\citep{yoran2024assistantbench}. SDP achieves the highest overall accuracy, precision, and EM using only a search API and URL-level scraping, without browser rendering, JavaScript, or form interaction. Its predicate structure decomposes each constraint into a separate verification step, and roughly 40\% of certified steps required at least one rejection-and-retry cycle. A recurring pattern in multi-entity tasks is that propagating partial findings with uncertainty markers outperforms forcing complete results, since validation pressure for completeness drives the agent to fill gaps with plausible but ungrounded values. SDP does not lead on Hard-tier accuracy, where Magentic-One~\citep{fourney2024magentic} wins, because Hard-tier tasks require multi-page browsing sessions a search-only interface cannot replicate. Where this limitation does not apply, the advantage is decisive: on Easy-tier tasks SDP achieves 92.8\%, exceeding the next best by over 10 points, confirming that structured validation rather than richer tool access is the binding factor.

\begin{table*}[h]
\centering
\caption{Results on AssistantBench~\cite{yoran2024assistantbench}. We report overall Accuracy, Precision, Exact Match (EM), and Accuracy stratified by difficulty (Easy/Medium/Hard).}
\label{tab:assistantbench}
\small
\setlength{\tabcolsep}{5pt}
\resizebox{0.66\textwidth}{!}{%
\begin{tabular}{@{}llcccccc@{}}
\toprule
& & & & & \multicolumn{3}{c}{Accuracy by Difficulty} \\
\cmidrule(lr){6-8}
Method & LLM & Accuracy & Precision & EM & Easy & Medium & Hard \\
\midrule
Infogent~\cite{reddy2025infogent}            & GPT-4o & 14.5 & 20.4 & 5.5 & 63.9 & 19.3 & 8.4 \\
RALM-1S$\to$CB~\cite{trivedi2023interleaving}      & GPT-4T & 19.5 & 21.0 & 6.1 & 81.3 & 35.0 & 7.3 \\
CB-1S~\cite{press2023measuring}               & GPT-4T & 22.2 & 24.8 & 8.3 & 67.8 & 49.7 & 4.2 \\
SeeAct$\to$CB~\cite{zheng2024gpt}       & GPT-4T & 23.4 & 26.1 & 9.4 & 82.0 & 47.7 & 7.1 \\
SPA-CB~\cite{yoran2024assistantbench}        & GPT-4T & 25.2 & 27.5 & 9.9 & 80.7 & 42.7 & 12.4 \\
Magentic-One~\cite{fourney2024magentic}      & GPT-4o & 25.3 & 25.3 & 11.0 & 69.9 & 35.6 & \textbf{16.9} \\
ACP-Domain Agents~\cite{bhardwaj2025agent}   & GPT-4o & 28.3 & 30.0 & 11.0 & 67.8 & 48.5 & 15.5 \\
\midrule
\textbf{SDP (Ours)} & GPT-4o & \textbf{31.8} & \textbf{32.0} & \textbf{14.9} & \textbf{92.8} & \textbf{54.9} & 16.0 \\
\bottomrule
\end{tabular}%
}
\end{table*}

        \textbf{ScienceWorld.}
            Table~\ref{tab:scienceworld} reports results on ScienceWorld~\citep{wang2022scienceworld}. Among training-free methods on the standard 30-task GPT-4 protocol, SDP achieves the highest overall score (59.16), improving over the previous best (Plan-and-Act~\citep{erdogan2025plan}) by 11.3 points. The gap widens with task length. On Long tasks SDP reaches 50.41, leading Plan-and-Act by 15.6 points and Reflexion by over 20 points, consistent with the expectation that explicit state tracking matters more as the horizon grows. The predicate structure also yields plan compression. The LLM's plan for a task like ``boil water'' uses 7 predicates that subsume the oracle's 36-step action sequence while preserving causal ordering, a compression derived from the task description alone without in-context trajectory examples. Plan granularity proved more consequential than length. Overly abstract predicates confused the validator, while overly fine-grained ones inflated LLM calls without improving accuracy, and predicates pitched at a single observable state change performed most reliably. The dominant failure mode here is not planning or action selection but inaccurate validation, analyzed further in Section~\ref{sec:anatomy}.

\begin{table*}[h]
\centering
\vspace{-0.5em}

\begin{minipage}[t]{0.43\textwidth}
\vspace{0pt}
\centering
\captionof{table}{Results on ScienceWorld~\cite{wang2022scienceworld} across task lengths, reported as average score. All methods are training-free with GPT-4 on the standard 30-task protocol.}
\label{tab:scienceworld}
\vspace{-0.2em}

\footnotesize
\setlength{\tabcolsep}{2.8pt}
\renewcommand{\arraystretch}{1.17}

\begin{tabular*}{\linewidth}{@{\extracolsep{\fill}}lcccc@{}}
\toprule
Method & Short & Medium & Long & Overall \\
\midrule
SayCan~\cite{ahn2022can}                 & 43.83 & 36.58 & 23.65 & 33.82 \\
ReAct~\cite{yao2022react}                & 48.79 & 44.01 & 21.07 & 36.43 \\
CoT~\cite{wei2022chain}                  & 49.54 & 47.87 & 23.09 & 39.23 \\
Reflexion~\cite{shinn2023reflexion}      & 71.47 & 35.43 & 30.17 & 45.34 \\
EVOAGENT~\cite{yuan2025evoagent}         & 48.67 & 36.17 & 11.38 & 30.42 \\
Plan-and-Act~\cite{erdogan2025plan}      & 60.52 & 46.43 & 34.77 & 47.86 \\
\midrule
\textbf{SDP (Ours)}                      & \textbf{73.72} & \textbf{53.50} & \textbf{50.41} & \textbf{59.16} \\
\bottomrule
\end{tabular*}
\end{minipage}%
\hfill
\begin{minipage}[t]{0.535\textwidth}
\vspace{0pt}
\centering
\captionof{table}{Results on multi-hop QA in the fullwiki open-domain setting. Both benchmarks use BM25 retrieval over a 5M-paragraph Wikipedia corpus and 1,000 validation examples.}
\label{tab:multihop-qa}
\vspace{-0.2em}

\footnotesize
\setlength{\tabcolsep}{1.6pt}
\renewcommand{\arraystretch}{0.98}

\begin{tabular*}{\linewidth}{@{\extracolsep{\fill}}llcccc@{}}
\toprule
& & \multicolumn{2}{c}{HotpotQA~\cite{yang2018hotpotqa}} & \multicolumn{2}{c}{MuSiQue~\cite{trivedi2022musique}} \\
\cmidrule(lr){3-4} \cmidrule(lr){5-6}
Method & LLM & EM & F1 & EM & F1 \\
\midrule
RankZephyr~\cite{pradeep2023rankzephyr} & GPT-4o & 34.7 & 35.0 & \phantom{0}8.6 & 12.8 \\
RankZephyr + CoT~\cite{pradeep2023rankzephyr} & GPT-4o & 34.0 & 34.4 & \phantom{0}9.4 & 13.3 \\
RankGPT~\cite{sun2023chatgpt} & GPT-4o & 34.6 & 35.3 & \phantom{0}9.5 & 13.5 \\
SETR-CoT \& IRI~\cite{lee2025shifting} & GPT-4o & 39.2 & 40.5 & 12.3 & 16.9 \\
IRCoT~\cite{trivedi2023interleaving} & GPT-3 & 49.3 & 60.7 & 34.2 & 43.8 \\
PRISM~\cite{nahid2025prism} & GPT-4o & 54.2 & 67.0 & 31.2 & 41.8 \\
\midrule
\textbf{SDP} & GPT-4o & \textbf{58.3} & \textbf{67.2} & \textbf{41.4} & \textbf{51.9} \\
\bottomrule
\end{tabular*}
\end{minipage}

\vspace{-0.5em}
\end{table*}
    
        \textbf{Multi-hop question answering.}
            Table~\ref{tab:multihop-qa} reports results on HotpotQA~\citep{yang2018hotpotqa} and MuSiQue~\citep{trivedi2022musique}. SDP's advantage is most pronounced on MuSiQue, where chains reach three or four hops, consistent with the expectation that explicit state tracking matters more as the reasoning horizon grows. The primary strength here is that \textsc{Validate} blocks hallucinated intermediate findings before they propagate. In HotpotQA's distractor setting, where gold paragraphs are guaranteed alongside irrelevant ones, this is decisive: validation rejects findings from distractor passages and the agent re-attempts with the correct paragraph, while reactive baselines lack per-hop verification and are vulnerable to plausible but incorrect evidence that silently corrupts the chain. The converse is the primary challenge: when the retriever fails to surface relevant evidence, \textsc{Validate} rejects each unsupported finding but the loop has no better paragraph to try, exhausting its budget. Roughly 76\% of MuSiQue failures trace to this retrieval bottleneck rather than planning or validation errors. SDP achieves these results using BM25 alone, whereas the strongest baselines employ hybrid retrieval; integrating learned retrievers into \textsc{Realize} is a natural extension the framework accommodates.
            
    \subsection{Anatomy of SDP trajectories}
		\label{sec:anatomy}

        The previous results evaluated SDP as a prompting strategy. This section examines the artifact it produces. The certified trajectory carries structured information that a reactive trajectory does not, namely whether single actions advanced multiple predicates, where the agent stalls, how much progress was made before failure, and how well the validator's verdicts agree with ground truth.

        \begin{figure*}[t]
            \centering
            \includegraphics[width=0.90\textwidth]{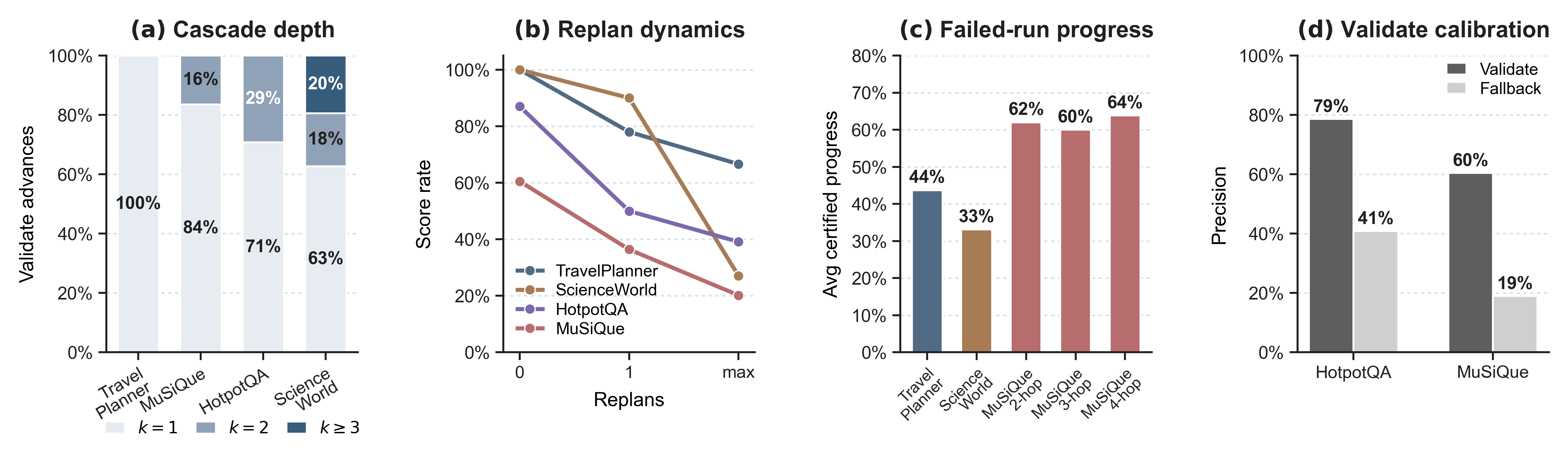}
            \vspace{-0.5em}
            \caption{Anatomy of SDP trajectories. \textbf{(a)} Cascade depth. \textbf{(b)} Score rate vs.\ replan count. \textbf{(c)} Certified progress in failed runs. \textbf{(d)} \textsc{Validate} calibration.}            
            \label{fig:anatomy}
            \vspace{-1.2em}
        \end{figure*}

        \textbf{Cascade and replan rates are environment properties (Figure~\ref{fig:anatomy}a,b).}
        	Whether a single action satisfies multiple predicates, and whether a failed plan is locally repairable, are environment-determined. Cascade ranges from 0\% on TravelPlanner, where each slot demands separate selection, to 37\% on ScienceWorld, where bundled sub-goals are common. The replan curve tracks environment recoverability. ScienceWorld holds full success through one replan, while TravelPlanner declines steadily as replanning rises, since replan cannot rescue infeasible option sets. This variation confirms SDP records these rates honestly rather than imposing a single pattern.
            
        \textbf{Where and how far runs fail (Figure~\ref{fig:anatomy}c).}
            Because every rejection attaches to a specific predicate, the trajectory records not only that a run failed but where and how far. TravelPlanner failures concentrate on hotel selection and return-transport; MuSiQue failures typically certify the first hop or two before stalling on a bridge entity. A reactive run is binary; an SDP run is a fraction. Failed TravelPlanner runs certify 44\% of their plan, and failed MuSiQue runs certify 60--64\% of their hops. The pattern is most revealing in MuSiQue, where certified hops grow with chain length, so longer chains fail farther in, surfacing more partial information rather than less.            
            
        \textbf{Calibration is auditable (Figure~\ref{fig:anatomy}d).}
        	Because \textsc{Validate} certifies the goal predicate as a discrete decision, every run yields a ground-truth check. The agreement between goal certification and a correct final answer is 79\% on HotpotQA and 60\% on MuSiQue. In both, runs that fall back to forced finalization show sharply lower precision (41\% and 19\% respectively), confirming that the certification signal carries information beyond LLM parametric guessing. The gap is not a hidden weakness but a measured one, since runs where \textsc{Validate} accepts the full chain but the answer is wrong form a small, identifiable subset whose patterns can guide tighter validator design.

    \subsection{Ablation study}
        \label{sec:ablation}
        
        The framework's three loop-level mechanisms, \textsc{Validate}, \textsc{Replan}, and cascade, are each removable. We replace \textsc{Validate} with always-pass, \textsc{Replan} with termination on budget exhaust, and cascade with a cap of $k \leq 1$. Scores are rescaled using conversion factors derived from each trajectory (action fidelity for \textsc{Validate}, certified-prefix ratio for \textsc{Replan}, both defined in Appendix~\ref{app:ablation}).  
        
        \begin{table*}[h]
        \centering
        \caption{\textbf{Ablation.} Effect of removing each loop-level mechanism.}
        \label{tab:ablation}
        \small
        \setlength{\tabcolsep}{6pt}
        \resizebox{0.90\textwidth}{!}{%
        \begin{tabular}{@{}lccccc@{}}
        \toprule
        & TravelPl. (Score) & SciWorld (AS) & AsstBench (EM) & HotpotQA (F1) & MuSiQue (F1) \\
        \midrule
        $-$\textsc{Validate} (always pass)        & 75.0 & 15.7 & 19.0 & 16.2 & 10.1 \\
        $-$\textsc{Replan} (terminate on budget)  & 74.7 & 27.5 & 25.0 & 49.6 & 36.6 \\
        $-$Cascade (force $k \leq 1$)             & 96.7 & 32.0 & 21.3 & 62.8 & 49.1 \\
        \midrule
        \textbf{Full SDP}                         & \textbf{96.7} & \textbf{59.2} & \textbf{31.8} & \textbf{67.2} & \textbf{51.9} \\
        \bottomrule
        \end{tabular}%
        }
        \end{table*}
                
        \textbf{\textsc{Validate} is the dominant contributor, except the environment substitutes for it.}
			Removing \textsc{Validate} produces the largest drop on four of five benchmarks, confirming it is the only mechanism catching uncertified findings before they propagate. TravelPlanner is the exception: its environment pre-filters candidates to exclude budget- and constraint-violating options, performing part of \textsc{Validate}'s role structurally. Without such a filter, \textsc{Validate} is the principal source of correctness.
        
        \textbf{\textsc{Replan} pays off in proportion to environment recoverability.}
        	\textsc{Replan} matters most where local re-attempts can recover progress and least where no alternative exists. ScienceWorld shows the strongest effect since most successful runs needed at least one replan to navigate an unexpected obstacle, while AssistantBench shows the weakest as nearly all tasks complete on the original plan. TravelPlanner presents a distinct pattern: the score drops not because replanning would have rescued the task, but because the certified prefix is shorter when the agent terminates on budget exhaustion.
            
        \textbf{Cascade is decisive only under tight attempt budgets.}
        	Cascade acts as a budget-multiplier, advancing the cursor by $k \geq 2$ and preventing budget exhaustion when actions bundle. ScienceWorld shows the largest effect without cascade (59.2 $\to$ 32.0), consistent with its 37\% cascade rate. On benchmarks where cascade is rare or the budget is generous, disabling it has little measurable impact.
        
    \section{Related work}
	\label{sec:related}

    \textbf{Reactive and reflective language agents.}
        The dominant paradigm maps observation history directly to action~\citep{yao2022react, schick2023toolformer}. Subsequent work adds verbal reflections~\citep{shinn2023reflexion, lippmann2025positive}, persistent memories~\citep{majumder2023clin}, or reusable skills~\citep{zhao2024expel, wang2023voyager}. These systems still share a structural gap: they do not construct states distinct from observations, certify transitions against those states, or produce trajectories whose Markov property can be stated. SDP instead produces a discrete certification at every step.

    \textbf{Action-planning agents.}
        A second line equips language agents with explicit planning, but plans over action sequences rather than state sequences~\citep{wang2023plan, yao2023tree, zhou2023language, li2026beyond, erdogan2025plan}. These approaches share the insight that committing to a plan before acting improves over reactive behavior, and SDP inherits it. The difference is what is planned. Action planners search $A^n$ for things to do, while SDP searches $\Sigma^n$ for conditions that should hold. Two consequences follow. First, plan entries are verifiable against observations, since predicates have truth values. Second, the plan is decoupled from the action space, so the same predicate chain can be realized by different actions without replanning.
        
    \textbf{World models and state abstraction.}
    	Several recent systems construct internal representations of the environment rather than raw observations, including temporal knowledge graphs~\citep{dinh2024reasonplanner}, free-form state descriptions~\citep{wang2023describe, sun2023adaplanner}, and feedback-triggered replanning~\citep{huang2022inner}. These systems share the intuition that an agent benefits from an explicit model of what is true in the world. The gap is in what follows from that model. In each case the internal representation is consumed by the action-producing module, with no per-step certification. SDP's four-operator decomposition enforces this separation architecturally rather than relying on the language model to maintain it implicitly.

    \textbf{Decision-theoretic formulations for language environments.}
        The standard approach posits an MDP or POMDP and learns or plans within it. This has been productive in dialogue systems~\citep{young2013pomdp, williams2007partially}, game playing~\citep{silver2017mastering}, and instruction following in gridworlds~\citep{branavan2009reinforcement}, where the state space is fixed in advance. Applying the same template to open-ended language environments requires fixing a state space and a transition kernel before the task is known---the obstacle of Section~\ref{sec:missing-inputs}. SDP sidesteps this by not assuming a pre-existing MDP. The agent constructs the state space at runtime, predicate by predicate, and the certified trajectory is the MDP, populated by the task rather than posited.            
    \section{Discussion}
	\label{sec:discussion}

	\textbf{Limitations.}
        SDP inherits the limits of its LLM-based operators. \textsc{Validate} can produce false positives on superficially matching observations, so the certified-state guarantee is only as strong as its soundness; Section~\ref{sec:anatomy} quantifies this through calibration. \textsc{Propose} similarly bounds plan quality by the LLM's ability to decompose goals into reachable predicates, and \textsc{Replan} cannot fully recover from consistently unreachable targets. Natural-language predicates also restrict which conditions can become states, especially for continuous quantities or conditions the LLM cannot articulate. Finally, SDP uses more LLM calls per environment step than reactive baselines.

    \textbf{What the framework makes possible.}
        In classical settings, MDP structure is valuable because it makes a class of operations well-defined, not because it enables one specific algorithm. SDP brings three such operations to language agents: per-predicate credit assignment through the integer $k$ from \textsc{Validate}, modular operator replacement through the fixed interfaces of Definition~\ref{def:sdp}, and formal progress measurement through the certified trajectory (Section~\ref{sec:anatomy}). Follow-up work can extend each direction by training a state generator for \textsc{Propose}, learning an action policy over $\Sigma$ for \textsc{Realize}, or applying offline RL to certified tuples $(s_t, a_t, k, s_{t+k})$. The present paper provides the structural foundation, and the fixed interfaces make these extensions natural.
    \section{Conclusion}
    \label{sec:conclusion}
    
    Language environments provide none of the runtime structure MDP analysis requires. No state space, no observation-to-state mapping, no certified transitions, and no termination criterion. SDP closes this gap by having the agent construct all four, predicate by predicate, producing a certified trajectory that carries exactly the objects downstream analysis presupposes. Experiments on five benchmarks confirm this structure improves task performance, with the advantage widening as the horizon grows, and supports analyses unavailable to reactive agents. SDP operates on the formal MDP and rigorously certifies each predicate, bringing decision-theoretic rigor to language agents.

    \section*{Acknowledgement}
        This work was supported in part by the DARPA Young Faculty Award, the National Science Foundation (NSF) under Grants \#2127780, \#2319198, \#2321840, \#2312517, and \#2235472, the Semiconductor Research Corporation (SRC), the Office of Naval Research through the Young Investigator Program Award, and Grants \#N00014-21-1-2225 and N00014-24-1-2547, Army Research Office Grant \#W911NF2410360. Additionally, support was provided by the Air Force Office of Scientific Research under Award \#FA9550-22-1-0253.
    {
        \small
        \bibliographystyle{plain}
        \bibliography{ref.bib}
    }    
    \newpage
    \appendix
    \section{Implementation details}
    \label{app:implementation_details}

    \subsection{Common Settings}
        \label{app:impl-common}
        
        All five benchmarks share the same SDP core loop; only the adapter layer differs per benchmark. This subsection describes the shared infrastructure.
        
        \paragraph{Core loop and adapter protocol.}    
            The SDP core executes a single algorithm regardless of benchmark: given a task, it calls the adapter's \textsc{Propose} to obtain an initial predicate plan, then iterates through the plan by calling \textsc{Realize} to produce an action for the current predicate and \textsc{Validate} to certify whether the action satisfies it. If validation succeeds with cascade count $k \geq 1$, the first $k$ predicates are removed from the plan and the state is updated. If validation fails, the attempt is recorded and \textsc{Realize} is retried up to the attempt budget. If the budget is exhausted, \textsc{Replan} generates a new plan tail from the stuck predicate onward. The loop terminates when the goal predicate is certified or the global step cap is reached.
        
        \paragraph{Attempt history and failure propagation.}
            Failed attempts are recorded in an append-only history that is threaded through subsequent \textsc{Realize} and \textsc{Replan} calls. Each entry contains the target predicate, the attempted action, the validation outcome, and the rejection reason. This history is surfaced in the LLM prompt so that the model avoids repeating previously failed choices. Crucially, failed attempts do not modify the certified state: only successful validations produce state updates. This separation preserves the MDP property that the state at any point reflects exactly the set of certified decisions.
        
        \paragraph{Validation modes.}
            SDP does not prescribe how validation is implemented; the framework supports both deterministic and LLM-based validators depending on the domain. In practice, we observe a spectrum across the five benchmarks. TravelPlanner uses fully deterministic validators (Python constraint checks over typed option objects), with no LLM involvement in validation at all. HotpotQA and MuSiQue use LLM-based verification (checking whether a finding is supported by a cited paragraph) combined with deterministic rejection filters (negation phrases, garbage answers, type mismatches). AssistantBench uses a three-phase pipeline (tool execution, LLM-based finding extraction, LLM-based verdict). ScienceWorld delegates validation entirely to the LLM after three hard-coded shortcuts (task completion, invalid action, goal-at-head). The choice of validation mode is an adapter-level design decision driven by domain structure: when constraints are formally specifiable (as in TravelPlanner), deterministic validation eliminates noise; when success criteria are semantic (as in ScienceWorld), LLM judgment is necessary.
        
        \paragraph{LLM backbone and inference.}
            Unless otherwise noted, all experiments use a single LLM backbone per experimental condition (the same model for \textsc{Propose}, \textsc{Realize}, \textsc{Validate}, and \textsc{Replan}). Temperature is set to 0 for all calls except where explicitly stated (AssistantBench finalize retries use temperature 0.3). All LLM outputs are requested in JSON format and parsed with a tolerant parser that strips code fences, attempts structured extraction, and falls back to regex-based field extraction. No few-shot examples are provided in any prompt; all prompts are zero-shot with detailed natural-language instructions.
        
    \subsection{TravelPlanner}
        \paragraph{State and predicate plan.}        
            The certified state $s_t$ contains task constants (origin, destination, days, travelers, budget, local constraints), a deterministically extracted trip topology (visited-city sequence, day-to-city mapping, inter-city travel days), and accumulated decisions (chosen transports, accommodations, per-day meals and attractions, running total cost, used restaurant/attraction names, collected cuisine types). All updates are immutable and the running cost is maintained as an invariant computed identically to the official evaluator's cost arithmetic.
        
            The state plan is a deterministic function of the task shape and city sequence. Eight predicate kinds are defined: (1)~outbound transport, (2)~accommodation (one per visited city), (3)~inter-city transport (one per consecutive city pair), (4)~day meals (one per day, covering all three slots), (5)~day attractions (one per day), (6)~return transport, (7)~budget check (aggregate: total cost $\leq$ budget and cuisine coverage), and (8)~render (assembly into official JSON format). The total count is $1 + n + (n{-}1) + 2N + 3$, yielding 11 predicates for 3d/1c, 17 for 5d/2c, and 23 for 7d/3c. Predicates are ordered so that cost-incurring global decisions precede per-day content, which precedes aggregate checks, ensuring that the remaining budget is well-defined at each filtering step.
        
        \paragraph{Constraint-upfront option filtering.}
            Before each \textsc{Realize} call, the adapter constructs the set of options that already satisfy every constraint the official evaluator will check. Five filters are applied in sequence: (1)~sandbox membership (the option appears in the reference data), (2)~local-constraint satisfaction (transport mode, room type, house rules), (3)~transport-mode consistency (self-driving must not co-occur with flight or taxi across the trip), (4)~trip-wide uniqueness (restaurants and attractions used on earlier days are excluded), and (5)~affordability (option cost plus running total $\leq$ budget). The filtered set is sorted cheapest-first and capped at 30 options. Because every presented option is guaranteed valid, any LLM selection passes validation, eliminating both malformed-output and constraint-violation failures.
            
            For meal predicates, the three slots are decided sequentially within a single predicate: after each slot, the state is updated so that the next slot's options exclude the just-chosen restaurant and reflect the updated budget. When the option set is empty, a diagnostic identifies which filter stage eliminated all candidates and surfaces this to the replan mechanism.
        
        \paragraph{Validation, replan, and hyperparameters.}
            All validators are deterministic and receive typed selections rather than raw LLM output; parsing is confined to the \textsc{Realize} step. The budget-check predicate is the only cross-predicate validator, confirming structural completeness, total cost, and cuisine coverage. Replan supports retry only: rollback to earlier predicates is not available because state updates are additive and not reversible. The attempt history is surfaced in the retry prompt so that the LLM avoids repeating failed choices.
            
            \begin{table}[h]
            \centering
            \small
            \caption{SDP configuration for TravelPlanner.}
            \label{tab:tp-hyperparams}
            \begin{tabular}{lrl}
            \toprule
            Parameter & Value & Description \\
            \midrule
            Attempt budget & 3 & Max \textsc{Realize} attempts per predicate before replan \\
            Max replans & 5 & Max replan cycles per predicate \\
            Global step cap & 60 & Max total steps across the episode \\
            Max options shown & 30 & Max options presented to the LLM per call \\
            \bottomrule
            \end{tabular}
            \end{table}

    \subsection{ScienceWorld}
        \label{app:impl-scienceworld}
        
        ScienceWorld~\citep{wang2022scienceworld} is an interactive text-based environment in which the agent performs multi-step science experiments (e.g., testing conductivity, boiling substances, growing plants). Tasks range from 5 to 50+ oracle steps, with partial-credit scoring that is milestone-based rather than monotonically increasing. The environment accepts free-form text actions and rejects unrecognized inputs with a fixed error string.
        
        ScienceWorld's instantiation of SDP differs from TravelPlanner in every component. Where TravelPlanner uses deterministic predicate plans and pre-filtered option sets, ScienceWorld requires LLM-generated plans, free-form action generation, and LLM-based validation. This contrast demonstrates that the SDP framework accommodates both structured-selection and open-ended-generation settings without changes to the core loop.

        \paragraph{State and predicate plan.}
            The state $s_t$ caches the latest environment observation, the full room description (which persists across actions regardless of the last command), the agent's inventory, the current score, a list of objects the agent has focused on, and a rolling action history. Unlike TravelPlanner, state updates are not purely additive: the environment itself is mutable, and the state snapshot reflects whatever the simulator reports after each action.
            
            All predicates are free-form natural-language condition assertions generated by the LLM at the start of each episode. A single unified predicate type replaces the eight-kind taxonomy (NAVIGATE, ACQUIRE, FOCUS, etc.) used in an earlier version of the adapter, which proved brittle: adding a new benchmark would have required defining a new kind set, new dispatch handlers, and new validation heuristics. Under the current design, the LLM produces declarative sub-goals such as ``The agent's inventory contains a metal pot'' or ``The water has been heated to a boil,'' and validation is delegated entirely to a separate LLM call. The only special-cased predicate is the terminal goal (``task score reaches 100''), which triggers a hard shortcut when the environment reports completion.
            
            A grounding mechanism prevents the LLM from hallucinating environment-specific names. The \textsc{Propose} prompt receives an explicit list of reachable rooms parsed from the simulator's valid-action set, and the LLM is instructed to reference only rooms from this list. During replanning, a cumulative list of targets that the engine has previously rejected is appended to the prompt so that the LLM avoids repeating proven-invalid references.
        
        \paragraph{Free-form action generation and LLM-based validation.}
            \textsc{Realize} issues a single LLM call that receives the current room description (the simulator's persistent state view, not the terse last-action response), inventory, the target predicate's description, recent action history, and the set of rooms already visited. The LLM returns a free-form action string; if the simulator rejects it, the attempt is recorded as a failure and the target is retried.
            
            \textsc{Validate} executes the action in the environment, then applies three hard shortcuts before resorting to an LLM call: (1)~if the environment reports task completion (done flag or score $\geq 100$), all remaining predicates are certified; (2)~if the engine rejects the action, the attempt fails with $k=0$ and the rejected target is cached for future replan context; (3)~if the goal predicate is at the head of the plan but the task is not complete, $k=0$. In all other cases, the LLM receives the target predicate description, the action, the resulting observation, the score delta, and a flag indicating whether the agent entered a previously unvisited room. The LLM returns a cascade count $k$ (how many consecutive head predicates are now satisfied) and a reason string. The cascade count is capped so that it cannot advance past natural-language predicates into the goal predicate.
            
            The visited-room flag addresses exploration sub-goals: predicates of the form ``The location of X is known to the agent'' are credited with $k \geq 1$ when the agent enters a new room, even if the target object has not yet been directly observed. Without this relaxation, exploration predicates became unreachable whenever the target was hidden inside a closed container in an unvisited room.
        
        \paragraph{Replan and hyperparameters.}
            Replanning reuses the \textsc{Propose} prompt with two additions: a context block describing the stuck predicate, the number of failed attempts, and the most recent action--failure pairs; and the cumulative invalid-targets list. The LLM generates a fresh sub-goal sequence from the current state.
            
            \begin{table}[h]
            \centering
            \small
            \caption{SDP configuration for ScienceWorld.}
            \label{tab:sw-hyperparams}
            \begin{tabular}{lrl}
            \toprule
            Parameter & Value & Description \\
            \midrule
            Attempt budget & 30 & Max \textsc{Realize} attempts per predicate before replan \\
            Max replans & 5 & Max replan cycles per predicate \\
            Global step cap & 500 & Max total steps across the episode \\
            \bottomrule
            \end{tabular}
            \end{table}
            
            The attempt budget is substantially higher than TravelPlanner's (30 vs.\ 3) because ScienceWorld actions are free-generated and frequently rejected by the engine, requiring more retries per predicate. The global step cap of 500 accommodates long tasks whose oracle trajectories already exceed 40 actions.
    
    \subsection{AssistantBench}
        \label{app:impl-assistantbench}
        
        AssistantBench~\citep{yoran2024assistantbench} is an open-ended web-information QA benchmark: given a natural-language question (e.g., ``List all gyms within half a mile of Tompkins Square Park''), the agent must search the web, scrape pages, and synthesize a final answer. Unlike ScienceWorld, there is no interactive environment---the agent orchestrates external tool calls (web search and page scraping) and submits a single answer.
        
        AssistantBench instantiates SDP as a \emph{narrowing pipeline}: the LLM decomposes the question into a sequence of evidence-gathering and reasoning steps, each of which is certified before the next begins. This design prevents a common failure mode of single-pass QA agents, in which partial evidence is silently dropped and the final answer omits qualifying candidates.
                
        \paragraph{State and predicate plan.}
            The state $s_t$ accumulates raw evidence and validated conclusions across the episode. It contains a search cache (mapping query strings to search-engine responses), a scrape cache (mapping URLs to page content), a list of certified step findings (each recording the step's goal, the evidence it produced, and its validated conclusion), a domain blacklist (sites that consistently fail scraping), and the current final answer.
        
            The predicate plan is generated by a \textsc{Decompose} call that breaks the question into up to five \emph{narrowing steps}, each labeled with one of four kinds: \emph{collect} (gather raw candidates), \emph{filter} (narrow an existing set by a criterion), \emph{extract} (pull a specific fact from known evidence), or \emph{compute} (derive the answer by calculation or comparison). After the step predicates, a \emph{finalize} predicate synthesizes the answer from accumulated findings, and a \emph{goal} predicate submits it.
        
        \paragraph{Tool-augmented realize and three-phase validation.}
            \textsc{Realize} for each step predicate consists of a planning LLM call that, given the question, prior certified findings, accumulated evidence, and previous failures on this step, produces a short list of tool actions (web searches and page scrapes). The adapter executes these actions via the Serper API, which provides both a search endpoint (returning ranked results with snippets) and a scraping endpoint (returning full-page markdown). Search results automatically trigger scraping of the top-ranked URL unless its domain appears on the blacklist. Domains are blacklisted at runtime when scraping fails (timeout, access denial, or content shorter than 100 characters), preventing repeated attempts against consistently unreachable sites such as TripAdvisor or Yelp.
        
            When all tool actions return only cached results (no fresh evidence), a fallback mechanism invokes GPT-4o with web-search to gather additional data. This fallback also fires after two consecutive validation failures on the same step, providing a qualitatively different evidence source.
            
            Validation proceeds in three phases. First, the tool actions are executed and their results are appended to the state's evidence caches. Second, a \textsc{Finding} LLM call reads the step's goal, the fresh evidence, and all prior certified findings, and produces a natural-language conclusion for this step. Third, a \textsc{Validate} LLM call judges whether the finding adequately addresses the step goal, returning a pass/fail verdict with a reason. If validation fails, the step's failure counter (maintained on the original state to persist across attempts) is incremented; after two consecutive rejections, the step is force-passed if the finding contains any substantive content, on the principle that partial data scores higher than an infinite retry loop.
        
        \paragraph{Replan, finalize, and hyperparameters.}
            Replanning generates new search queries via an LLM call that receives the stuck predicate, a diagnosis of what has been tried, and a strategy-shift prompt. It also proposes scraping of URLs from existing search results that have not yet been fetched, subject to domain-blacklist filtering. The original step predicates are re-added to give them another attempt with the newly gathered evidence.
            
            The finalize predicate synthesizes the final answer from all accumulated evidence. If the answer is weak (empty, null, or a refusal string), the adapter retries with an elevated temperature and an explicit instruction to return closest matches. After three weak attempts, a force-accept mechanism produces the best available answer rather than returning nothing.
            
            \begin{table}[h]
            \centering
            \small
            \caption{SDP configuration for AssistantBench.}
            \label{tab:ab-hyperparams}
            \begin{tabular}{lrl}
            \toprule
            Parameter & Value & Description \\
            \midrule
            Attempt budget & 3 & Max \textsc{Realize} attempts per predicate before replan \\
            Max replans & 2 & Max replan cycles per predicate \\
            Global step cap & 30 & Max total steps across the episode \\
            Max subqueries & 5 & Max decomposition steps per question \\
            \bottomrule
            \end{tabular}
            \end{table}
            
            The conservative step cap reflects that each step involves multiple external API calls (search plus auto-scrape plus potential fallback), making per-step cost substantially higher than in the interactive benchmarks.

    \subsection{HotpotQA}
        \label{app:impl-hotpotqa}
        
        HotpotQA~\citep{yang2018hotpotqa} is a multi-hop question answering benchmark in which each question requires reasoning over two Wikipedia paragraphs. We evaluate under two settings: \emph{distractor}, where ten paragraphs (two gold, eight distractors) are provided with the question, and \emph{open-domain}, where the agent retrieves evidence from a BM25 index over the full Wikipedia corpus.
        
        \paragraph{State, predicate plan, and retrieval.}
            The state $s_t$ contains the question, all available paragraphs (provided in distractor mode or accumulated via retrieval in open-domain mode), a list of certified hop findings (each recording the sub-question, the supporting paragraph title, the extracted finding, and a bridge entity for chaining), a search cache mapping query strings to retrieved paragraphs, and the current final answer.
            
            The predicate plan follows a fixed four-predicate structure: two \emph{hop} predicates (one per sub-question), an \emph{answer} predicate that synthesizes the final response from certified findings, and a \emph{goal} predicate. The two sub-questions are generated by an LLM decomposition call at the start of each episode, and the bridge entity extracted from the first hop is available to the second hop's realize call, enabling the characteristic two-hop reasoning chain.
            
            In the open-domain setting, \textsc{Realize} for each hop first generates a search query via a dedicated LLM call that sees the original question, the current sub-question, prior findings, and previously tried queries. The query is executed against a BM25 index built over approximately 5.2 million Wikipedia abstract articles, with title boosting: a two-stage retrieval first matches query n-grams against paragraph titles (with prefix matching to handle Wikipedia disambiguation suffixes such as ``(film)'' or ``(album)''), then falls back to standard BM25 content scoring. Retrieved paragraphs accumulate in the search cache across hops, so the second hop can draw on evidence gathered during the first.
        
        \paragraph{Two-stage validation with paragraph fallback.}
            \textsc{Validate} for hop predicates applies a two-stage verification process to address the most frequent failure mode: the LLM cites the correct finding but attributes it to a wrong or generic paragraph title. In the primary stage, the adapter resolves the cited title against available paragraphs (exact match, then substring match, then search-cache lookup) and asks an LLM verifier whether the finding is supported by the resolved paragraph's text and correctly answers the sub-question. Findings that contain explicit negation phrases (``does not provide,'' ``not mentioned,'' ``not specified'') are rejected before the LLM call.
            
            If the primary stage fails---either because the cited title cannot be resolved or because the LLM verifier rejects the attribution---a fallback stage scans all accumulated paragraphs for one that supports the finding. Candidate paragraphs are ranked by keyword overlap with the finding (extracting content words of four or more characters, excluding stopwords, and counting matches in each paragraph's text). The top candidate is then submitted to the same LLM verifier. If confirmed, the finding is accepted with the corrected title. This fallback recovers cases where the LLM extracted the right fact from the right paragraph but hallucinated the paragraph's title.
            
            Answer-predicate validation rejects garbage responses (``unknown,'' ``not available,'' ``cannot determine'') and, when the answer predicate is immediately followed by the goal predicate, certifies both in a single cascade ($k = 2$). Replanning re-decomposes the question with failure context while preserving any hops already certified, so only the remaining sub-questions are regenerated.
        
        \paragraph{Hyperparameters.}
            \begin{table}[h]
            \centering
            \small
            \caption{SDP configuration for HotpotQA.}
            \label{tab:hqa-hyperparams}
            \begin{tabular}{lrl}
            \toprule
            Parameter & Value & Description \\
            \midrule
            Attempt budget & 3 & Max \textsc{Realize} attempts per predicate before replan \\
            Max replans & 2 & Max replan cycles per predicate \\
            Global step cap & 20 & Max total steps across the episode \\
            BM25 top-$k$ & 10 & Retrieved paragraphs per search query (open-domain) \\
            \bottomrule
            \end{tabular}
            \end{table}
            
            The low global step cap reflects the fixed plan length (four predicates): most episodes complete in 4--8 realize calls, with additional steps consumed only by retries on validation failure or replan-triggered re-decomposition.

    \subsection{MuSiQue}
        \label{app:impl-musique}
        
        MuSiQue~\citep{trivedi2022musique} extends multi-hop QA to 2--4 hop compositional chains. The SDP adapter shares the same architecture as HotpotQA (Section~\ref{app:impl-hotpotqa})---LLM-based decomposition, hop-by-hop evidence extraction with bridge-entity chaining, LLM-based hop validation, and BM25 retrieval for the open-domain setting---with three adaptations.
        
        First, the predicate plan is dynamic: the hop count is parsed from each task's identifier, yielding plans of 4--6 predicates (2--4 hops, plus answer and goal). The decomposition prompt requests the corresponding number of sub-questions, and replanning preserves already-certified hops and regenerates only the remaining tail.
        
        Second, the open-domain retriever operates over a corpus of approximately 84,000 unique paragraphs collected from all MuSiQue splits following the IRCoT protocol~\citep{trivedi2023interleaving}, rather than the 5.2-million-article Wikipedia corpus used for HotpotQA. The retriever uses the same BM25 architecture with title boosting and prefix matching. To handle the longer reasoning chains, the adapter employs escalating search diversity: after four unsuccessful queries, it falls back to entity-only queries extracted from the sub-question; after six, it asks the LLM to guess the answer and searches for the guess directly.
        
        Third, answer validation is stricter to match MuSiQue's short-span answer format. Beyond the garbage-answer rejection shared with HotpotQA, the adapter rejects self-referential answers (entities already present in the question), intermediate bridge entities that are not the final answer, answers exceeding 15 words, and question--answer type mismatches (e.g., a non-numeric answer to a ``when'' question). A final LLM sanity check verifies that the proposed answer is consistent with the certified hop chain before acceptance.
        
        \begin{table}[h]
        \centering
        \small
        \caption{SDP configuration for MuSiQue.}
        \label{tab:mq-hyperparams}
        \begin{tabular}{lrl}
        \toprule
        Parameter & Value & Description \\
        \midrule
        Attempt budget & 3 & Max \textsc{Realize} attempts per predicate before replan \\
        Max replans & 2 & Max replan cycles per predicate \\
        Global step cap & 25 & Max total steps across the episode \\
        BM25 top-$k$ & 10 & Retrieved paragraphs per search query (open-domain) \\
        Corpus size & ${\sim}$84K & Unique paragraphs from train + validation splits \\
        \bottomrule
        \end{tabular}
        \end{table}
        
        The global step cap is set slightly higher than HotpotQA's (25 vs.\ 20) to accommodate 3- and 4-hop tasks, which require proportionally more realize calls.
    \section{Benchmarks}
    \label{app:benchmarks}

    \subsection{TravelPlanner}        
        \paragraph{Benchmark selection.}
        	TravelPlanner~\citep{xie2024travelplanner} is a constraint-satisfaction benchmark in which an agent must assemble a multi-day travel itinerary that jointly satisfies hard constraints (budget, transportation mode, accommodation type) and commonsense constraints (meal variety, attraction count). We selected it because these constraints map directly onto SDP predicates. Each constraint becomes a separate predicate that must be certified before the plan advances, making TravelPlanner a natural testbed for whether per-predicate certification prevents the constraint violations that dominate plan-as-text baselines. The benchmark spans multiple task scales (3-day/1-city through 7-day/3-city), allowing us to observe how SDP behaves as plan complexity grows.
        
        \paragraph{Evaluation protocol.}
        	We evaluate on all 180 validation tasks using the tool-use setting, in which the agent queries search APIs (flights, hotels, restaurants, attractions, ground transport) to discover options before assembling the itinerary. Evaluation follows the official TravelPlanner scoring pipeline, reporting Delivery Rate, Commonsense constraint satisfaction (Micro/Macro), Hard Constraint satisfaction (Micro/Macro), and Final Pass Rate.
        
        \paragraph{Baseline selection.}
        	We include nine baselines spanning four categories. ReAct~\citep{yao2022react}, Reflexion~\citep{shinn2023reflexion}, Plan-and-Solve~\citep{wang2023plan}, GoalAct~\citep{chen2025enhancing}, and TDP~\citep{li2026beyond} are reproduced by us using Gemini-3.1-flash-lite on all 180 tasks with the same evaluation pipeline. MIRROR~\citep{guo2025mirror} (GPT-4o), EvoAgent~\citep{yuan2025evoagent}, PMC~\citep{zhang2025planning}, and ATLAS~\citep{choi2025atlas} (all Gemini-2.5-Pro) report numbers from their original papers on the same 180-task validation split and official scoring. SDP is evaluated with both Gemini-3.1-flash-lite and GPT-4o to enable fair comparison at both scales. The Gemini-3.1-flash-lite runs allow direct comparison against the five reproduced baselines under identical conditions, while the GPT-4o runs allow comparison against MIRROR at the same backbone scale.

    \subsection{ScienceWorld}        
        \paragraph{Benchmark selection.}
        	ScienceWorld~\citep{wang2022scienceworld} is an interactive text-based simulator in which an agent must perform scientific experiments by navigating rooms, manipulating objects, and reasoning about physical processes. We selected it because its tasks are inherently long-horizon (oracle action sequences reach 36 steps) and require strict causal ordering, making it a natural testbed for whether SDP's per-predicate certification and replan mechanism improve performance as the horizon grows. The benchmark groups its 30 tasks into Short, Medium, and Long categories, enabling direct measurement of how each method scales with task length.
        
        \paragraph{Evaluation protocol.}
        	We follow the standard 30-task evaluation protocol established by SwiftSage~\citep{lin2023swiftsage} and adopted by subsequent training-free methods. All methods use GPT-4 as the backbone. Performance is reported as average score across Short, Medium, Long, and Overall.
        
        \paragraph{Baseline selection.}
        	We include six training-free baselines. SayCan~\citep{ahn2022can}, ReAct~\citep{yao2022react}, CoT~\citep{wei2022chain}, Reflexion~\citep{shinn2023reflexion}, EVOAGENT~\citep{yuan2025evoagent}, and Plan-and-Act~\citep{erdogan2025plan} all report results using GPT-4 on the same 30-task protocol. Numbers are taken from their original papers. CoT, Plan-and-Act report only Overall scores. We do not include TDP~\citep{li2026beyond} or RPMS in the table because their evaluation protocols differ from the standard 30-task setup, making direct comparison unreliable.

    \subsection{AssistantBench}
        \paragraph{Benchmark selection.}
        	AssistantBench~\citep{yoran2024assistantbench} is an open-ended web reasoning benchmark consisting of 214 tasks that require retrieving, filtering, and integrating information from multiple web sources. We selected it because a large class of its tasks involves filtering entities through several independent constraints, a structure that maps naturally onto SDP predicates. Each constraint becomes a separate verification step, and \textsc{Validate} blocks ungrounded findings before they propagate to subsequent steps. The benchmark also stratifies tasks by difficulty (Easy/Medium/Hard), enabling fine-grained analysis of where SDP's advantage concentrates and where tool limitations become binding.
        
        \paragraph{Evaluation protocol.}
        	We evaluate on the full test split (181 tasks) via submission to the official HuggingFace leaderboard. The agent uses Serper search API and URL-level scraping as its tool interface, without browser rendering, JavaScript execution, or form interaction. Scoring follows the official AssistantBench metrics, reporting Accuracy, Precision, Exact Match (EM), and Accuracy stratified by difficulty.
        
        \paragraph{Baseline selection.}
        	We include seven baselines from the official AssistantBench leaderboard. Infogent~\citep{reddy2025infogent}, Magentic-One~\citep{fourney2024magentic}, and ACP-Domain Agents~\citep{bhardwaj2025agent} use GPT-4o. RALM-1S$\to$CB~\citep{trivedi2023interleaving}, CB-1S~\citep{press2023measuring}, SeeAct$\to$CB~\citep{zheng2024gpt}, and SPA-CB~\citep{yoran2024assistantbench} use GPT-4-Turbo. All numbers are taken from the official leaderboard under the same test split and scoring pipeline. SDP uses GPT-4o, enabling direct comparison against the three GPT-4o baselines. The tool interface differs across methods. Magentic-One employs a dedicated browser-controlling sub-agent, while SDP relies solely on search API and scraping. This difference is most visible on Hard-tier tasks, where multi-page browsing sessions are required.

    \subsection{HotpotQA and MuSiQue}        
        \paragraph{Benchmark selection.}
        	HotpotQA~\citep{yang2018hotpotqa} and MuSiQue~\citep{trivedi2022musique} are multi-hop question answering benchmarks that require chaining evidence across multiple Wikipedia paragraphs. We selected them because multi-hop reasoning creates a natural chain of predicates, one per hop, where each hop's finding feeds into the next. This structure makes \textsc{Validate}'s per-hop verification directly consequential. HotpotQA requires two hops, while MuSiQue extends to two, three, and four hops, enabling measurement of how SDP scales with reasoning depth. We use the answerable subset of MuSiQue as defined by~\cite{trivedi2022musique}.
        
        \paragraph{Evaluation protocol.}
        	Both benchmarks are evaluated in the fullwiki (open-domain) setting. The agent retrieves evidence from a 5M-paragraph Wikipedia corpus using BM25 retrieval, matching the setup of IRCoT~\citep{trivedi2023interleaving} and PRISM~\citep{nahid2025prism}. We evaluate on 1,000 validation examples for each benchmark. Performance is reported as Exact Match (EM) and token-level F1. SDP uses GPT-4o as the backbone.
        
        \paragraph{Baseline selection.}
        	We include six baselines from three families. RankZephyr, RankZephyr + CoT~\citep{pradeep2023rankzephyr}, and RankGPT~\citep{sun2023chatgpt} represent retrieve-then-read approaches with reranking. SETR-CoT \& IRI~\citep{lee2025shifting} extends this with iterative retrieval. IRCoT~\citep{trivedi2023interleaving} interleaves retrieval with chain-of-thought reasoning. PRISM~\citep{nahid2025prism} uses agentic retrieval with multi-step planning. All numbers are taken from their original papers. IRCoT uses GPT-3 while all other baselines and SDP use GPT-4o. SDP uses BM25 retrieval alone, whereas PRISM employs hybrid retrieval (BM25 + learned reranker), making the comparison conservative for SDP.
    \section{Ablation Details}
    \label{app:ablation}
    
    The ablation estimates in Table~\ref{tab:ablation} are computed by replaying recorded trajectories under counterfactual rules rather than re-executing the environment. Each counterfactual modifies one mechanism and rescales the observed score using a conversion factor derived from the trajectory itself.
    
    \paragraph{Action fidelity (for $-$\textsc{Validate}).}
    	When \textsc{Validate} is replaced by always-pass, every proposed predicate is accepted regardless of whether the observation actually supports it. We define \emph{action fidelity} as the fraction of \textsc{Realize} attempts in the original trajectory whose action would have produced a correct outcome without validation. Formally, for a trajectory with $m$ certified steps, action fidelity is $f = m_{\text{correct-on-first}} / m$, where $m_{\text{correct-on-first}}$ counts the steps whose first \textsc{Realize} attempt was subsequently certified by \textsc{Validate}. The ablated score is the original score multiplied by $f$, reflecting the expected degradation when incorrect findings propagate unchecked.
    
    \paragraph{Certified-prefix ratio (for $-$\textsc{Replan}).}
    	When \textsc{Replan} is removed, the agent terminates as soon as the per-target attempt budget $b$ is exhausted at any predicate. We define the \emph{certified-prefix ratio} as $r = t_{\text{stall}} / n$, where $t_{\text{stall}}$ is the cursor position at which the first budget exhaustion occurs and $n$ is the plan length. For runs that complete without triggering \textsc{Replan}, $r = 1$ and the score is unchanged. For runs that required replanning, the ablated score is the original score multiplied by $r$, reflecting the fraction of the plan that would have been certified before termination.
    
    \paragraph{Cascade ($-$Cascade).}
    	When cascade is disabled by capping \textsc{Validate} at $k \leq 1$, each action certifies at most one predicate. The additional steps that would have been required are counted from the original trajectory by summing $(k_i - 1)$ over all steps where $k_i \geq 2$. If the total step count under this rule exceeds the agent's budget, the run is marked as incomplete and scored proportionally. On benchmarks where cascade rarely occurs (e.g., TravelPlanner with $k = 1$ on all steps), the ablated score equals the original.
    
    \paragraph{Limitations of replay estimation.}
    	Replay estimates assume that removing a mechanism does not alter the distribution of subsequent actions or observations. In practice, an unchecked error at step $t$ may cause downstream actions to diverge from the recorded trajectory. The estimates therefore represent an optimistic bound on the true ablation effect, since cascading errors would likely degrade performance further. The direction and relative ordering of the estimates across benchmarks are nonetheless informative for identifying which mechanism contributes most in each environment.
\end{document}